\documentclass{article}
\usepackage{ijcai22}

\usepackage{times}

\usepackage{soul}
\usepackage{url}
\usepackage[hidelinks]{hyperref}
\usepackage[utf8]{inputenc}
\usepackage[small]{caption}
\usepackage{graphicx}
\usepackage{amsmath}
\usepackage{booktabs}
\usepackage{amssymb}
\usepackage[normalem]{ulem}
\usepackage{MnSymbol}
\usepackage{stmaryrd}
\usepackage{graphicx}
\usepackage{color,xcolor}
\usepackage{amsthm}

\newtheorem{definition}{Definition}
\newtheorem{proposition}{Proposition}

\newtheorem{example}{Example}

\newcommand{\oplans}[2]{\ensuremath{\operatorname{Oplans}(#1,#2)}}
\newcommand{\concl}[1]{\ensuremath{\operatorname{concl}(#1)}}

\title{Value-based  Practical Reasoning: Modal Logic + Argumentation}

\author{
Jieting Luo$^1$
\and
Beishui Liao$^2$\and
Dov Gabbay$^3$
\affiliations
$^1$University of Bern\\
$^2$Zhejiang University\\
$^3$King's College London
\emails
jieting.luo@inf.unibe.ch,
baiseliao@zju.edu.cn,
dov.gabbay@kcl.ac.uk
}

\begin{document}

\maketitle

\begin{abstract}
Autonomous agents are supposed to be able to finish tasks or achieve goals that are assigned by their users through performing a sequence of actions. Since there might exist multiple plans that an agent can follow and each plan might promote or demote different values along each action, the agent should be able to resolve the conflicts between them and evaluate which plan he should follow. In this paper, we develop a logic-based framework that combines modal logic and argumentation for value-based practical reasoning with plans. Modal logic is used as a technique to represent and verify whether a plan with its local properties of value promotion or demotion can be followed to achieve an agent's goal. We then propose an argumentation-based approach that allows an agent to reason about his plans in the form of supporting or objecting to a plan using the verification results.
\end{abstract}

\section{Introduction}
Autonomous agents are supposed to be able to perform value-based ethical reasoning based on their value systems in order to distinguish moral from immoral behavior. Existing work on value-based practical reasoning such as \cite{atkinson2018taking}\cite{bench2012using} \cite{liao2021representation} demonstrates how an agent can reason about what he should do among alternative action options that are associated with value promotion or demotion. However, agents are supposed to be able to finish tasks or achieve goals that are assigned by their users through performing a sequence of actions. Since there might exist multiple plans that an agent can follow and each plan might promote or demote different values along each action, the agent should be able to resolve the conflicts between them and evaluate which plan he should follow. If the decision-making problem concerns choosing a plan instead of an action, then we first need to know how an agent can see whether he can follow a plan to achieve his goal. Verification approaches that are developed based on modal logic only allow us to verify whether a goal can be achieved under specific conditions such as norm compliance assumptions \cite{aagotnes2007normative}\cite{knobbout2012reasoning}\cite{alechina2013reasoning}, namely telling us whether a plan works or not, but cannot tell us what we should do. For sure, we can collect the verification results regarding whether a plan promotes or demotes a specific set of values and then compare different plans using lifting approaches as what has been done in \cite{luo2019formal}. However, the order lifting problem is a major challenge in many areas of AI and no approach is ultimately “correct”. Moreover, the agent in our setting needs to lift the preference over values to the preference over plans with respect to value promotion and demotion, which even complicates the problem. Therefore, we need a more natural and intuitive approach. It has been shown that argumentation provides a useful mechanism to model and resolve conflicts \cite{dung1995acceptability}, and particularly can be used for the decision-making of artificial intelligence and provides explanation for that \cite{prakken2006combining}\cite{amgoud2007formalizing}. In this paper, we develop a logic-based framework that combines modal logic and argumentation for value-based practical reasoning with plans. Modal logic is used as a technique to represent and verify whether a plan with its local properties of value promotion or demotion can be followed to achieve an agent's goal. Using the verification results to construct arguments, we then propose an argumentation-based approach that allows an agent to reason about his plans in the form of support and objection without using lifting approaches. We prove several formal properties to characterize our approach, indicating it is consistent with our rationality of decision-making.

\section{Logical Framework}
The semantic structure of this paper is a transition system that represents the computational behavior of a system caused by an agent's actions in the agent's subjective view. It is basically a directed graph where a set of vertexes $S$ corresponds to possible states of the system, and the relation $\rightarrow \subseteq S \times Act \times S$ represents the possible transitions of the system. When a certain action $\alpha \in Act$ is performed, the system might progress from a state $s$ to a different state $s^\prime$ in which different propositions hold. Formally, 
\begin{definition}[Transition Systems]
Let $\Phi =\{p,q,...\}$ be a finite set of atomic propositional variables, a transition system is a tuple $T=(S,Act,\rightarrow,V)$ over $\Phi$, where 
\begin{itemize}
\item $S$ is a finite, non-empty set of states;
\item $Act$ is a finite, non-empty set of actions;
\item $\rightarrow \subseteq S \times Act \times S$ is a transition relation between states with actions, which we refer to as the transition relation labeled with an action; we require that for all $s \in S$ there exists an action $a \in Act$ and a state $s^\prime \in S$ such that $(s,a,s^\prime) \in \rightarrow$; we restrict actions to be deterministic, that is, if $(s,a,s^\prime) \in \rightarrow$ and $(s,a,s^{\prime\prime}) \in \rightarrow$, then $s^\prime = s^{\prime\prime}$; since the relation is partially functional, we write $s[\alpha]$ to denote the state $s^\prime$ for which it holds that $(s, \alpha, s^\prime) \in \rightarrow$; we also use $s[\alpha_1, \ldots, \alpha_n]$ to denote the resulting state for which a sequence of actions $\alpha_1, \ldots, \alpha_n$ succinctly execute from state $s$;
\item $V$ is a propositional valuation $V: S \to 2^{\Phi}$ that assigns each state with a subset of propositions which are true at state $s$; thus for each $s \in S$ we have $V(s) \subseteq \Phi$.
\end{itemize}
\end{definition}
Note that the model is deterministic: the same action performed in the same state will always result in the same resulting state. A pointed transition system is a pair $(T, s)$ such that $T$ is a transition system, and $s \in S$ is a state from $T$. Adopted from \cite{knobbout2016dynamic}\cite{knobbout2014reasoning}, the language $\mathcal{L}_{\Box}$ is propositional logic extended with action modality. Formally, its grammar is defined below:
\[\varphi ::= p \mid \lnot \varphi \mid \varphi \lor \varphi \mid \Box_{\alpha}\varphi \quad (p \in \Phi, \alpha \in Act) \]
Given a pointed transition system $(T,s)$, we define the semantics with respect to the satisfaction relation $\models$ inductively as follows:
\begin{itemize}
\item $T, s \models p$ iff $p \in V(s)$;
\item $T, s \models \lnot \varphi$ iff $T, s \not \models \varphi$;
\item $T, s \models \varphi \lor \psi$ iff $T, s \models \varphi$ or $T, s \models \psi$;
\item $T, s \models \Box_{\alpha} \varphi$ iff $s[\alpha] \models \varphi$.
\end{itemize}
The remaining classical logic connectives are assumed to be defined as abbreviations in terms of $\lnot$ and $\lor$ in the conventional manner. Given a pointed transition system $(T, s)$, we say that a sequence of actions $\alpha_1 \ldots \alpha_n$ brings about a $\varphi$-state if and only if $T, s \models \Box_{\alpha_1} \ldots \Box_{\alpha_n} \varphi$. As standard, we write $T \models \varphi$ if $T, s \models \varphi$ for all $s \in S$, and $\models \varphi$ if $T \models \varphi$ for all $T$. 

A transition system represents how a system progresses by an agent's actions. Besides, an agent in the system is assumed to have his own goal, which is a formula expressed in propositional logic $\mathcal{L}_{prop}$. It is indeed possible for an agent to have multiple goals and his preference over different goals. For example, a goal hierarchy is defined in \cite{aagotnes2007normative} to represent increasingly desired properties that the agent wishes to hold. However, we find that the setting about whether the agent has a goal or multiple goals is in fact not essential for our analysis, so we simply assume that the agent only has a goal for simplifying our presentation.

\begin{example}
\begin{figure}
  \centering
    \includegraphics[width=0.3\textwidth]{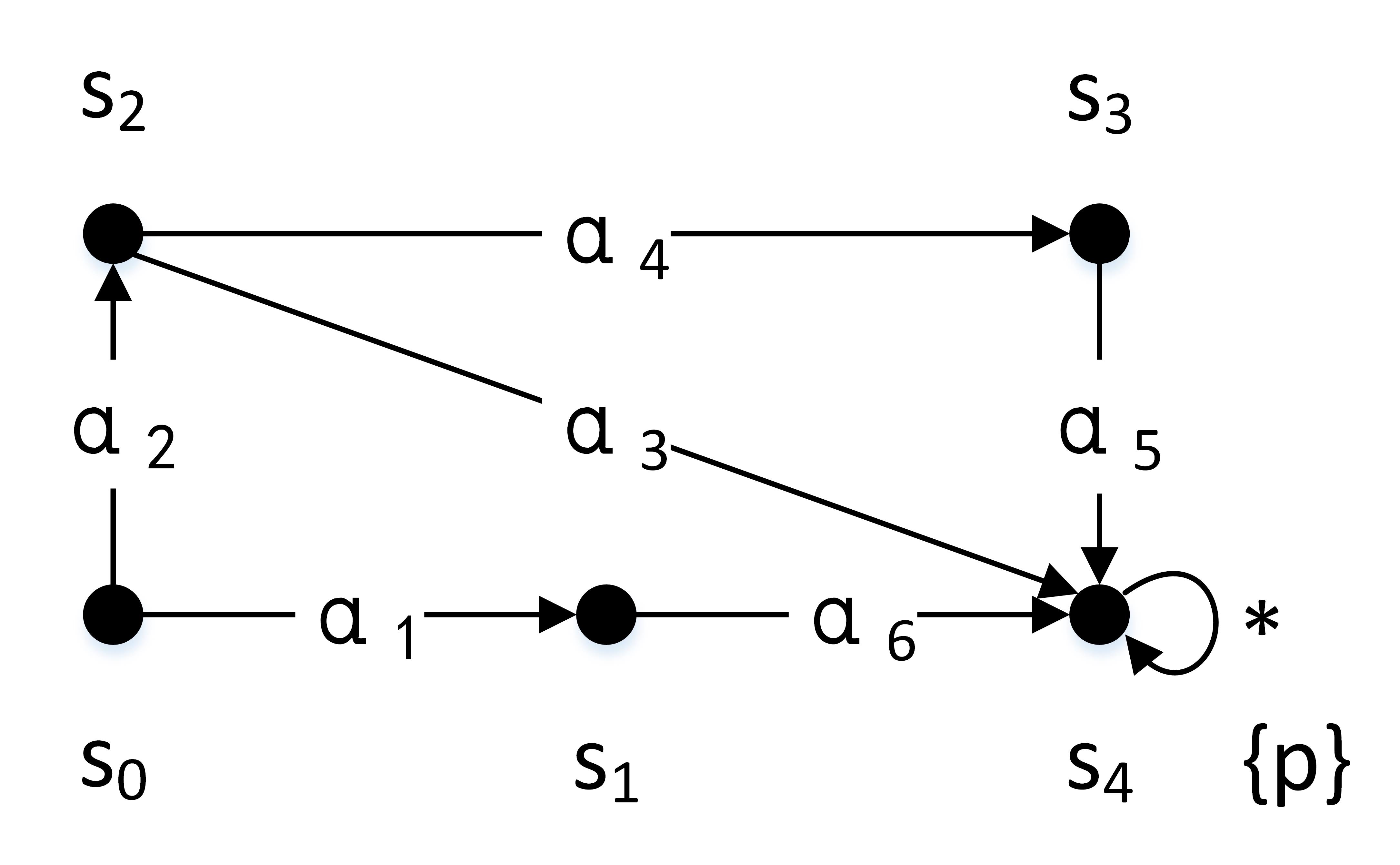}
    \caption{Transition system $T$.}\label{Drawing1}
\end{figure}

Consider the transition system $T$ in Fig. \ref{Drawing1}, which represents how an agent can get to a pharmacy to buy medicine for his user. State $s_0$ is the initial state, representing staying at home, and proposition $p$, representing arriving at a pharmacy, holds in state $s_4$. The agent can perform actions $\alpha_1$ to $\alpha_6$ in order to get to state $s_4$. From this transition system, the following formulas hold:
\[ T, s_0 \models \Box_{\alpha_1} \Box_{\alpha_6} p,\]
\[ T, s_0 \models \Box_{\alpha_2} \Box_{\alpha_3} p,\]
\[ T, s_0 \models \Box_{\alpha_2} \Box_{\alpha_4} \Box_{\alpha_5} p,\]
which means that the agent can first perform action $\alpha_1$ and then action $\alpha_6$, or action $\alpha_2$ followed by action $\alpha_4$, or action $\alpha_2$ followed by actions $\alpha_4$ and $\alpha_5$, to get to the pharmacy.
\end{example}

It is important for an agent not only to achieve his goal, but also to think about how to achieve his goal. As we can see from the running example, there are multiple ways for the agent to get to the pharmacy, and the agent needs to evaluate which one is the best to choose. In this paper, agents are able to perform value-based practical reasoning in terms of planning their actions to achieve their goals. We first assume that an agent has a set of values. A value can be seen as an abstract standard according to which agents have their preferences over options. For instance, if we have a value denoting \emph{equality}, we prefer the options where equal sharing or equal rewarding hold. Unlike \cite{luo2019formal} where a value is interpreted as a state formula, we simply assume a value as a primitive structure without considering how it is defined. We assume that agents can always compare any two values, so we define an agent's value system as a total pre-order (instead of a strict total order) over a set of values, representing the degree of importance of something. 
\begin{definition}[Value System]
A value system $V = (\operatorname{Val},\precsim)$ is a tuple consisting of a finite set of values $\operatorname{Val} = \{v, ..., v'\}$ together with a total pre-ordering $\precsim$ over $\operatorname{Val}$. When $v \precsim v'$, we say that value $v'$ is at least as important as value $v$. As is standard, we define $v \sim v^\prime$ to mean $v \precsim v^\prime$ and $v^\prime \precsim v$, and $v \prec v^\prime$ to mean $v \precsim v^\prime$ and $v \not \sim v^\prime$.
\end{definition}
We label some of the transitions with the values promoted and demoted by moving from a starting state to a ending state. Notice that not every transition can be labeled, as some transitions may not be relevant to any value in an agent's value system. Formally, function $\delta: \{+, -\} \times \operatorname{Val} \to 2^{\to}$ is a valuation function which defines the status (promoted (+) or demoted (-)) of a value $v \in \operatorname{Val}$ ascribed to a set of transitions. We then define a value-based transition system $VT$ as a transition system together with a value system $V$ and a function $\delta$.
\begin{definition}[Value-based Transition Systems]
A value-based transition system is defined by a triple $VT = (T, V, \delta)$, where $T$ is a transition system, $V$ is a value system and $\delta$ is a valuation function that assigns value promotion or demotion to a set of transitions.
\end{definition}
Given a sequence of actions with respect to a value-based transition system, we then express whether the performance of the sequence in a state promotes or demotes a specific value, which can be done by extending our language. Given a pointed value-based transition system $(VT, s)$ and a value $v \in \operatorname{Val}$, the satisfaction relation $VT, s \models \psi$ is extended with the following new semantics:
\begin{itemize}
\item $VT, s \models_{+v} \Box\alpha_1, \ldots, \Box\alpha_n \varphi$ iff $s[\alpha_1, \ldots, \alpha_n] \models \varphi$ and there exists $1 \leq m \leq n$ such that $(s[\alpha_1, \ldots, \alpha_{m-1}], \alpha_m, s[\alpha_1, \ldots, \alpha_{m}]) \in \delta(+,v)$; 
\item $VT, s \models_{-v} \Box\alpha_1, \ldots, \Box\alpha_n \varphi$ iff $s[\alpha_1, \ldots, \alpha_n] \models \varphi$ and there exists $1 \leq m \leq n$ such that $(s[\alpha_1, \ldots, \alpha_{m-1}], \alpha_m, s[\alpha_1, \ldots, \alpha_{m}]) \in \delta(-,v)$.
\end{itemize}
The formula $VT, s \models_{+v} \Box\alpha_1, \ldots, \Box\alpha_n \varphi$ (resp. $VT, s \models_{-v} \Box\alpha_1, \ldots, \Box\alpha_n \varphi$) should be intuitively read as $\varphi$ is achieved after the performance of a sequence of actions $\alpha_1, \ldots, \alpha_n$ in state $s$ and there exists an action that promotes (resp. demotes) value $v$ in the sequence. Notice that the formula only expresses the local property of a sequence of actions in terms of value promotion or demotion by an action within the sequence. Thus, it is possible that an action within the sequence promotes value $v$ but it gets demoted by another action within the sequence, meaning that both $VT, s \models_{+v} \Box\alpha_1, \ldots, \Box\alpha_n \varphi$ and $VT, s \models_{-v} \Box\alpha_1, \ldots, \Box\alpha_n \varphi$ hold at the same time. Through checking the above formulas, the agent is then aware of whether he can perform the sequence of actions to achieve his goal and which value gets promoted or demoted along the sequence. We continue our running example to illustrate how to use our logical language to express and verify properties of sequences of actions.

\begin{example}\label{list}
\begin{figure}
  \centering
    \includegraphics[width=0.35\textwidth]{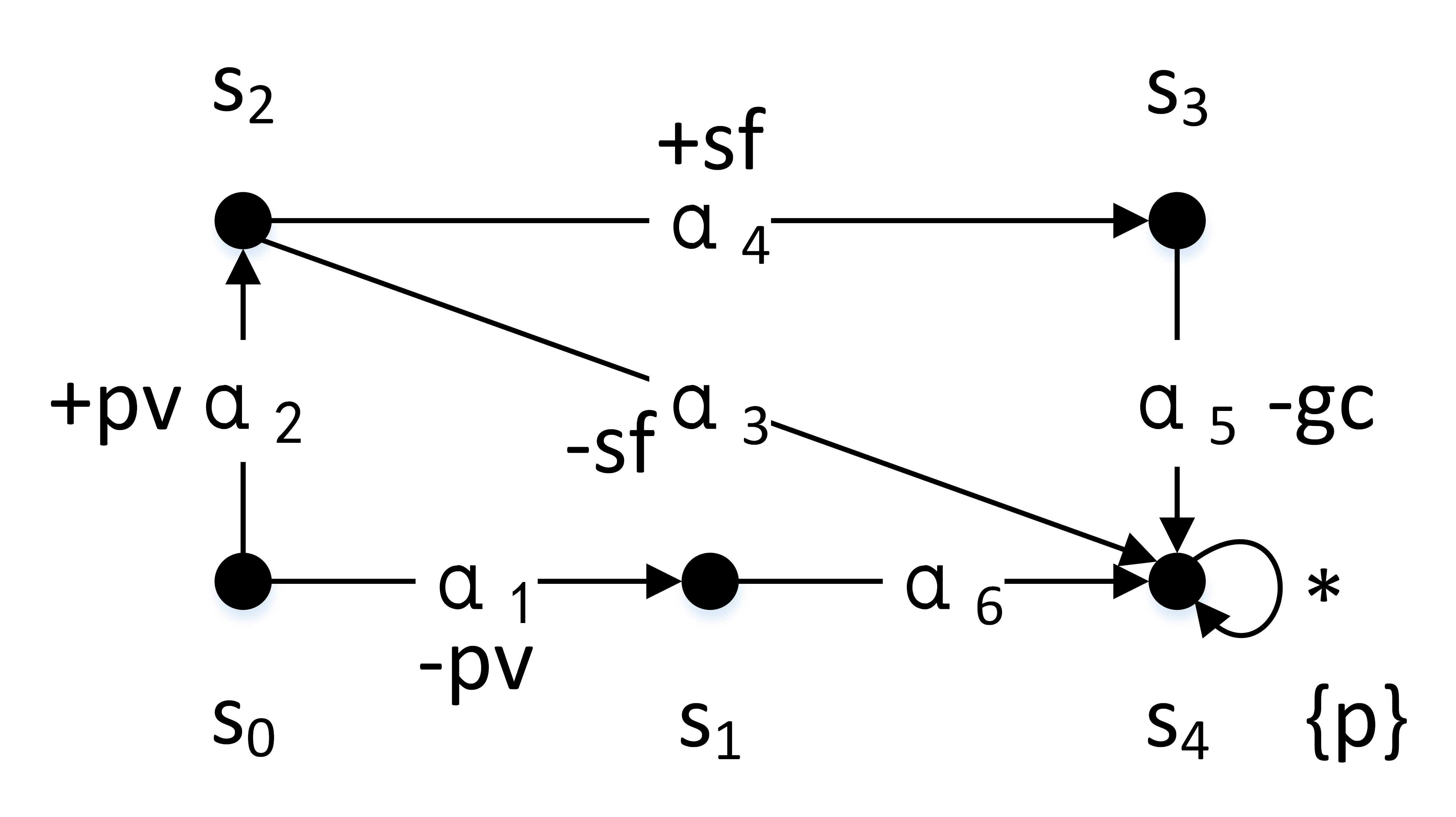}
    \caption{A Value-based transition system $VT$.}\label{Drawing3}
\end{figure}

Suppose the ethical agent has privacy ($pv$), safety ($sf$) and good conditions ($gc$) as his values and a value system as $pv \prec gc \prec sf$. As in Figure \ref{Drawing3}, some of the transitions have been labeled with value promotion or demotion with respect to the agent's values. Taking action $\alpha_1$ in state $s_0$ is interpreted as going through the neighbor's garden for taking shortcut, which demotes the value of privacy of the neighbor, conversely taking action $\alpha_2$ in state $s_0$ is interpreted as stepping on a normal way, which promotes the value of privacy of the neighbor. Taking action $\alpha_3$ means crossing the road without using the crosswalk, which demotes the value of safety of the agent, and conversely taking action $\alpha_4$ in state $s_2$ promotes the value of safety of the agent. Finally, performing action $\alpha_5$ in state $s_3$ means stepping into water. As the agent is a robot, which should avoid getting wet, this choice will demote the value of maintaining good conditions of the agent. The agent can verify whether he can achieve his goal while promoting or demoting a specific value by performing a sequence of actions. The verification results are listed below:
\begin{align*}
&T, s_0 \models_{-pv} \Box_{\alpha_1} \Box_{\alpha_6} p \\
&T, s_0 \models_{+pv} \Box_{\alpha_2} \Box_{\alpha_3} p \\
&T, s_0 \models_{-sf} \Box_{\alpha_2} \Box_{\alpha_3} p \\
&T, s_0 \models_{+pv} \Box_{\alpha_2} \Box_{\alpha_4} \Box_{\alpha_5} p \\
&T, s_0 \models_{+sf} \Box_{\alpha_2} \Box_{\alpha_4} \Box_{\alpha_5} p \\
&T, s_0 \models_{-gc} \Box_{\alpha_2} \Box_{\alpha_4} \Box_{\alpha_5} p 
\end{align*}
\end{example}

\section{Decision-making: an Argumentation-based Approach}
Given a transition system and an agent's goal, model checking and verification techniques allow us to verify whether an agent can achieve his goal while promoting or demoting a specific value by performing a sequence of actions. Since following different plans might promote or demote different sets of values, next question is how the agent decides what to do given the verification results. In this paper, we propose to use argumentation as a technique for an agent's decision-making. Formal argumentation is a nonmonotonic formalism for representing and reasoning about conflicts based on the construction and the evaluation of interacting arguments \cite{dung1995acceptability}. In particular, it has been used in practical reasoning, which is concerned by reasoning about what agents should do, given different alternatives and outcomes they bring about \cite{bench2012using}\cite{amgoud2007formalizing}. We first define the notion of plans. A plan is defined as a finite sequence of actions that are enabled by our underlying transition system. Formally,
\begin{definition}[Plans]
Given a pointed value-based transition system $(VT,s)$ and a formula $p \in \mathcal{L}_{prop}$ as an agent's goal, a plan is defined as a finite sequence of actions over $Act$, denoted as $\lambda = (\alpha_1, \alpha_2, \ldots, \alpha_n)$, such that $VT, s \models \Box \alpha_1 \Box \alpha_2 \ldots \Box \alpha_n p$. 
\end{definition}
The definition is equivalent to saying that a plan is a sequence of actions, each of which can be performed succinctly with respect to the pointed value-based transition system, to achieve the agent's goal. The agent has to reason about the available plans with respect to their goal achievement and value promotion or demotion. In order to do that, it is intuitive to define an argument as a plan together with its local property of value promotion or demotion. Based on the verification results, we define two types of arguments.
\begin{definition}[Ordinary Arguments and Blocking Arguments]\label{pargument}
Given a pointed value-based transition system $(VT,s)$, a formula $p \in \mathcal{L}_{prop}$ as an agent's goal, a plan $\lambda = (\alpha_1, \alpha_2, \ldots, \alpha_n)$ and $v \in \operatorname{Val}$, an ordinary argument is a pair $\langle +v, \lambda \rangle$ such that $VT, s \models_{+v} \Box_{\alpha_1} \Box_{\alpha_2} \ldots \Box_{\alpha_n} p$; a blocking argument is a pair $\langle -v, \lnot \lambda \rangle$ such that $VT, s \models_{-v} \Box_{\alpha_1} \Box_{\alpha_2} \ldots \Box_{\alpha_n} p$; we use $\mathcal{A}_o$ (resp. $\mathcal{A}_b$) to denote the set of ordinary arguments (resp. blocking arguments), and $\mathcal{A}= \mathcal{A}_o \cup \mathcal{A}_b$ to denote the set of two types of arguments.
\end{definition}
Both an ordinary argument and a blocking argument correspond to a verification result. An ordinary argument $\langle +v, \lambda \rangle$ is interpreted as ``the agent should follow plan $\lambda$ to achieve his goal because it promotes a value $v$'', which supports the performance of plan $\lambda$, and a blocking argument $\langle -v, \lnot \lambda \rangle$ is interpreted as ``the agent should not follow plan $\lambda$ to achieve his goal because it demotes a value $v$'', which objects to the performance of plan $\lambda$. Conventionally, we might represent an argument using an alphabet ($a, b, \ldots$) if we do not care about the internal structure of the argument. 
\begin{example}
From the verification results listed in Example \ref{list}, the agent can construct the following ordinary arguments and blocking arguments: $\langle -pv, \lnot (\alpha_1, \alpha_6) \rangle$, $\langle +pv, (\alpha_2, \alpha_3) \rangle$, $\langle -sf, \lnot (\alpha_2, \alpha_3) \rangle$, $\langle +pv, (\alpha_2, \alpha_4, \alpha_5) \rangle$, $\langle +sf, (\alpha_2, \alpha_4, \alpha_5) \rangle$ and $\langle -gc, \lnot (\alpha_2, \alpha_4, \alpha_5) \rangle$.
\end{example}

When we get to choose a plan to follow, there are conflicts between the alternatives as they cannot be followed all at the same time. The conflicts are interpreted as attacks between two ordinary arguments supporting different plans and one ordinary argument and one blocking argument supporting and objecting to the same plan respectively in this paper.
\begin{definition}[Attacks]\label{pattack}
Given a set of ordinary arguments $\mathcal{A}_o$ and a set of blocking arguments $\mathcal{A}_b$, 
\begin{itemize}
\item for any two ordinary arguments $\langle +v_a, \lambda_a \rangle, \langle +v_b, \lambda_b \rangle \in \mathcal{A}_o$, $\langle +v_a, \lambda_a \rangle$ attacks $\langle +v_b, \lambda_b \rangle$ iff $\lambda_a \not = \lambda_b$;
\item for any ordinary argument $\langle +v_a, \lambda_a \rangle \in \mathcal{A}_o$ and any blocking argument $\langle -v_b, \lnot \lambda_b \rangle \in \mathcal{A}_b$, 
\begin{itemize}
\item $\langle +v_a, \lambda_a \rangle$ attacks $\langle -v_b, \lnot \lambda_b \rangle$ iff $\lambda_a = \lambda_b$;
\item $\langle -v_b, \lnot \lambda_b \rangle$ attacks $\langle +v_a, \lambda_a \rangle$ iff $\lambda_a = \lambda_b$.
\end{itemize}
\end{itemize}
The set of attacks over $\mathcal{A}$ are denoted as $\mathcal{R}$.
\end{definition}
It is obvious that our attack relation is mutual. It should be noticed that there is no attack between two blocking arguments, as a blocking argument only functions as blocking the conclusion of an ordinary argument but does not make a conclusion by itself. 
\begin{example}
The running example has three ordinary arguments and three blocking arguments. The attack relation is depicted in Figure \ref{Drawing2}, where any two ordinary arguments with different plans are mutually attacked, and any ordinary argument and blocking argument with the same plan are mutually attacked.

\begin{figure}
  \centering
    \includegraphics[width=0.5\textwidth]{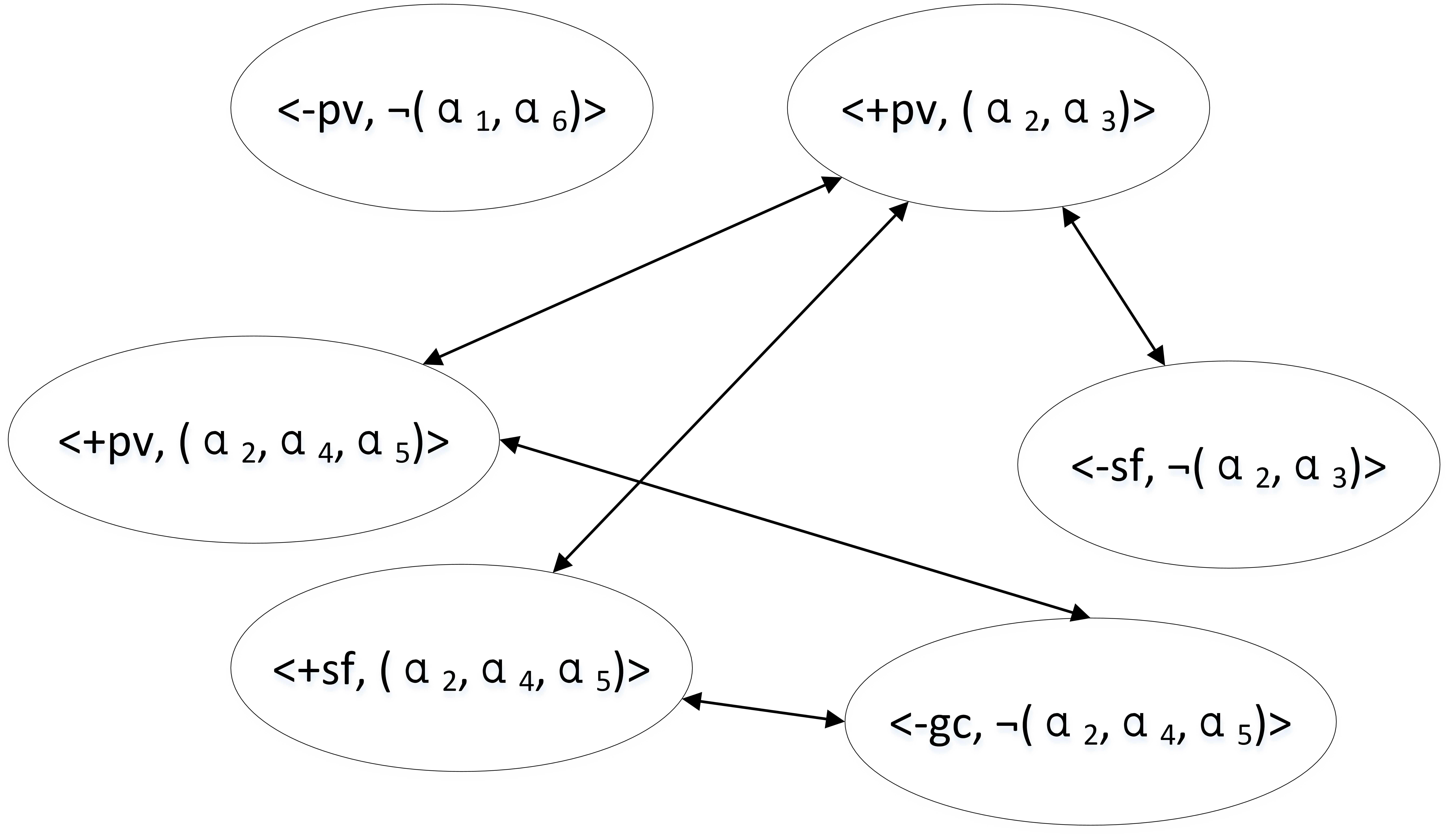}
    \caption{Attack relation between ordinary arguments and blocking arguments.}\label{Drawing2}
\end{figure}

\end{example}
The attack relation represents conflicts between plans. However, the notion of attack may not be sufficient for modeling conflicts between arguments, as an agent has his preference over the values that are promoted or demoted by different plans. In structured argumentation frameworks such as $ASPIC^+$ \cite{modgil2013general}, an argument $a$ can be used as a counter-argument to another argument $b$, if $a$ successfully attacks, i.e. defeats, $b$. Whether an attack from $a$ to $b$ (on its sub-argument $b^\prime$) succeeds as a defeat, may depend on the relative strengths of $a$ and $b$, i.e. whether $a$ is strictly stronger than, or strictly preferred over $b^\prime$. For this paper, recall that an agent has a value system, which was defined as a total pre-order over a set of values. We can then determine the preference over two arguments with respect to value promotion and demotion based on the value system. The notion of defeats combines the notions of attack and preference.
\begin{definition}[Defeats]\label{pdefeat}
Given a set of ordinary arguments $\mathcal{A}_o$ and a set of blocking arguments $\mathcal{A}_b$, a set of attacks $\mathcal{R}$ over $\mathcal{A}$ and a value system $V$, 
\begin{itemize}
\item for any two ordinary arguments $\langle +v_a, \lambda_a \rangle, \langle +v_b, \lambda_b \rangle \in \mathcal{A}_o$, $\langle +v_a, \lambda_a \rangle$ defeats $\langle +v_b, \lambda_b \rangle$ iff $\langle +v_a, \lambda_a \rangle$ attacks $\langle +v_b, \lambda_b \rangle$ and $v_a \not \prec v_b$;
\item for any ordinary argument $\langle +v_a, \lambda_a \rangle \in \mathcal{A}_o$ and any blocking argument $\langle -v_b, \lnot \lambda_b \rangle \in \mathcal{A}_b$, 
\begin{itemize}
\item $\langle +v_a, \lambda_a \rangle$ defeats $\langle -v_b, \lnot \lambda_b \rangle$ iff $\langle +v_a, \lambda_a \rangle$ attacks $\langle -v_b, \lnot \lambda_b \rangle$ and $v_a \not \prec v_b$;
\item $\langle -v_b, \lnot \lambda_b \rangle$ defeats $\langle +v_a, \lambda_a \rangle$ iff $\langle -v_b, \lnot \lambda_b \rangle$ attacks $\langle +v_a, \lambda_a \rangle$ and $v_b \not \prec v_a$.
\end{itemize}
\end{itemize}
The set of defeats over $\mathcal{A}$ based on an attack relation and a value system are denoted as $\mathcal{D}(\mathcal{R}, V)$. We write $\mathcal{D}$ for short if it is clear from the context.
\end{definition}
In words, given mutual attacks between two arguments, the attack from the argument with less preferred value to the attack from the argument with a more preferred value does not succeed as a defeat. One might ask whether it is more convenient to combine the notions of attack relation and defeat relation. We argue that two notions represent the relation between two arguments from different perspectives, one for the conflicts between plans and the other for the preferences over values. Because of that, defining these two notions separately can make our framework more clear, even though technically it is possible to combine them. It is obvious to see that our defeat relation can form a \emph{two-length cycle} in which two arguments have equivalent or the same values. 
\begin{proposition}
Given two ordinary arguments $\langle +v_a, \lambda_a \rangle, \langle +v_b, \lambda_b \rangle \in \mathcal{A}_o$, $\langle +v_a, \lambda_a \rangle$ and $\langle +v_b, \lambda_b \rangle$ form a two-length cycle iff $\lambda_a \not= \lambda_b$ and ($v_a =  v_b$ or $v_a \sim v_b$). Given an ordinary argument $\langle +v_a, \lambda_a \rangle \in \mathcal{A}_o$ and a blocking argument $\langle -v_b, \lnot \lambda_b \rangle \in \mathcal{A}_b$, $\langle +v_a, \lambda_a \rangle$ and $\langle -v_b, \lnot \lambda_b \rangle$ form a two-length cycle iff  $\lambda_a = \lambda_b$ and ($v_a =  v_b$ or $v_a \sim v_b$).
\end{proposition}
\begin{proof}
Proof follows from Definition \ref{pdefeat} directly.
\end{proof}
However, we have the result that our defeat relation obeys the property of irreflexivity and never forms any odd cycle.
\begin{proposition}\label{odd}
Given a set of arguments $\mathcal{A}$, a defeat relation $\mathcal{D}$ on $\mathcal{A}$ never forms any odd cycle.
\end{proposition}
\begin{proof}
According to Definition \ref{pdefeat}, in order for an argument $a$ to defeat another argument $b$, the value $v_a$ that belongs to $a$ must be not less preferred than $v_b$ that belongs to $b$. Since an agent's value system is a total pre-order over a set of values, arguments can only form a cycle in which any two arguments are mutually defeated with the values involved are equivalent or the same. So $\mathcal{D}$ on $\mathcal{A}$ never forms any odd cycle.
\end{proof}
\begin{proposition}
Given a set of arguments $\mathcal{A}$, a defeat relation $\mathcal{D}$ on $\mathcal{A}$ is irreflexive.
\end{proposition}
\begin{proof}
It is a special case of Proposition \ref{odd} for the number of arguments in the odd cycle being one.
\end{proof}
We are now ready to construct a Dung-style abstract argumentation framework with ordinary arguments, blocking arguments and the defeat relation on them.
\begin{definition}[Plan-based Argumentation Frameworks]
Given a pointed value-based transition system $(VT,s)$ and a formula $p \in \mathcal{L}_{prop}$ as an agent's goal, a plan-based argumentation framework over $(VT,s)$ and $p$ is a pair $PAF = (\mathcal{A}, \mathcal{D})$, where $\mathcal{A}$ is a set of arguments and $\mathcal{D}$ is a defeat relation on $\mathcal{A}$.
\end{definition}

\begin{example}
In our running example, the agent has a value system as $pv \prec gc \prec sf$, which means that safety is more important than keeping good condition, and keeping good condition is more important than privacy. We then can see some of the attacks in Figure \ref{Drawing2} do not succeed as defeats. For example, argument $\langle +pv, (\alpha_2, \alpha_4, \alpha_5) \rangle$ and argument $\langle -gc, \lnot (\alpha_2, \alpha_4, \alpha_5) \rangle$ are mutually attacked, but since $pv$ is less preferred than $gc$, only the attack from argument $\langle -gc, \lnot (\alpha_2, \alpha_4, \alpha_5) \rangle$ to argument $\langle +pv, (\alpha_2, \alpha_4, \alpha_5) \rangle$ becomes a defeat. For arguments $\langle +pv, (\alpha_2, \alpha_4, \alpha_5) \rangle$ and $\langle +pv, (\alpha_2, \alpha_3) \rangle$, since $pv = pv$, the mutual attacks between them succeed as mutual defeats. For the space limitation, we do not analyze all the defeats, which is depicted in Figure \ref{Drawing4}. Interestingly, argument $\langle -pv, \lnot (\alpha_1, \alpha_6)$ does not receive any defeats or defeat any arguments not because $pv$ is the most preferred value, but because there is no ordinary argument with plan $(\alpha_1, \alpha_6)$.

\begin{figure}
  \centering
    \includegraphics[width=0.5\textwidth]{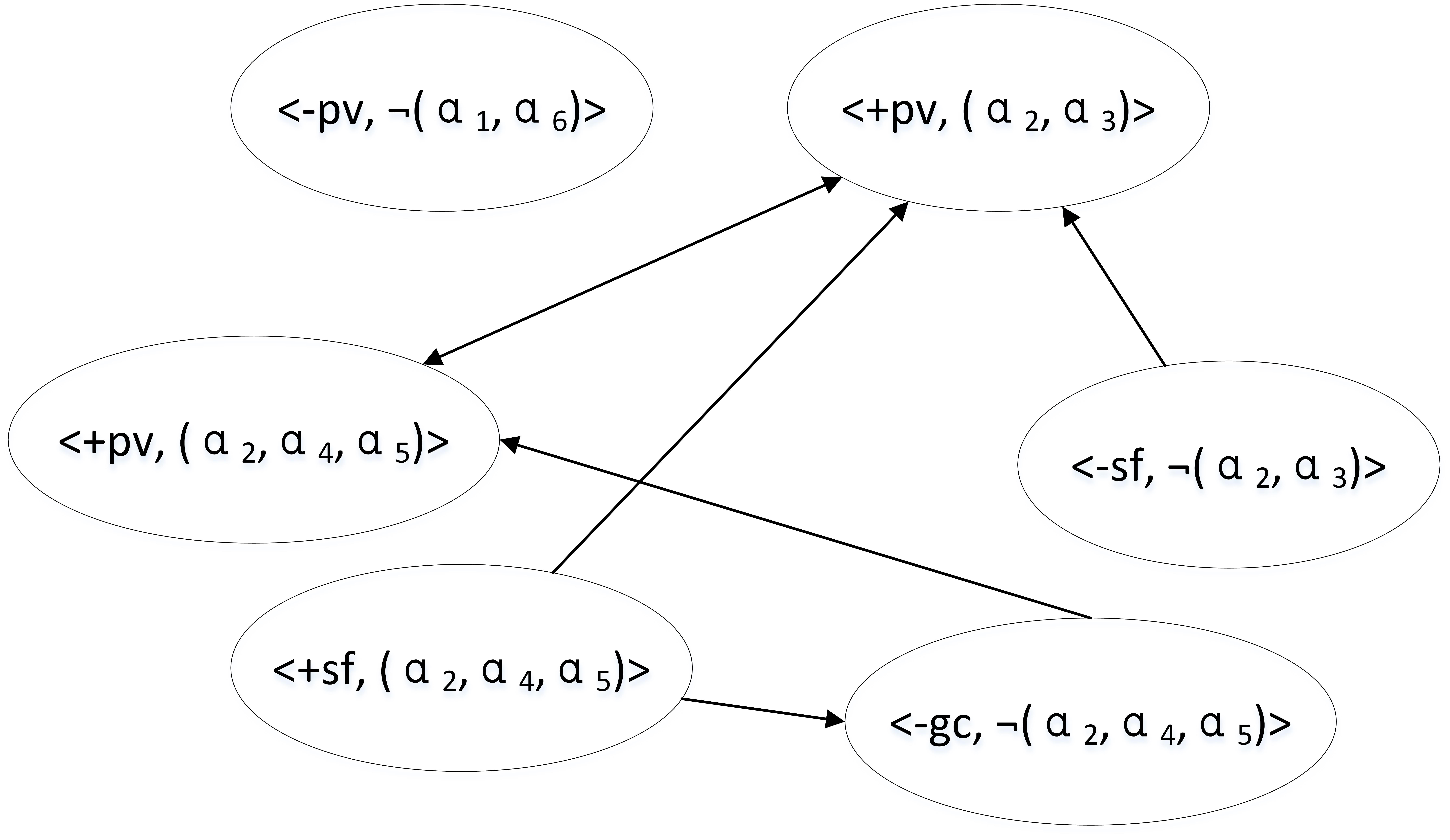}
    \caption{An argumentation framework.}\label{Drawing4}
\end{figure}

\end{example}

Given a plan-based argumentation framework $PAF$, status of arguments is evaluated, producing sets of arguments that are acceptable together, which are based on the notions of conflict-freeness, acceptability and admissibility. The well-known argumentation semantics are listed as follows, each of which provides a pre-defined criterion for determining the acceptability of arguments in a $PAF$ \cite{dung1995acceptability}.
\begin{definition}[Conflict-freeness, Acceptability, Admissibility and Extensions]\label{PAFsemantics}
Given $PAF=(\mathcal{A}, \mathcal{D})$ and $ E \subseteq \mathcal{A}$,
		\begin{itemize}
		\item A set $ E$ of arguments is \textit{conflict-free} iff there does not exist $a, b \in E$ such that $(a, b)\in \mathcal{D}$.
		\item An argument $a \in  \mathcal{A}$ is \textit{acceptable} w.r.t. a set $ E$ ($a$ is defended by $ E$), iff $\forall (b,a)\in  \mathcal{D}$,  $\exists c \in  E$ such that $(c,b)\in  \mathcal{D}$.
		\item A conflict-free set of arguments $ E$ is \textit{admissible} iff each argument in $E$ is acceptable w.r.t. $ E$.			
		\item $E$ is a complete extension of $PAF$ iff $E$ is admissible and each argument in $\mathcal{A}$ that is acceptable w.r.t. $E$ is in $E$.
		\item $E$ is the grounded extension of $PAF$ iff $E$ is the minimal (w.r.t. set inclusion) complete extension.
		\item $E$ is the preferred extension of $PAF$ iff $E$ is a maximal (w.r.t. set inclusion) complete extension.
		\item $E$ is a stable extension of $PAF$ iff $E$ is conflict-free and $\forall b \in \mathcal{A}\backslash E$, $\exists a \in E$ such that $(a,b)\in \mathcal{D}$.
		\end{itemize}
\end{definition}
We use $sem \in \{cmp, prf, grd, stb\}$ to denote the complete, preferred, grounded and stable semantics, respectively, and $\mathcal{E}_{sem}(PAF)$ to denote the set of extensions of $PAF$ under a semantics in $sem$. The following propositions characterize our argumentation framework in terms of Dung's semantics.
\begin{proposition}
Given $PAF=(\mathcal{A}, \mathcal{D})$, $\mathcal{E}_{prf}(PAF) = \mathcal{E}_{stb}(PAF)$.
\end{proposition}
\begin{proof}
Since our defeat relation $\mathcal{D}$ never forms an odd cycle by Proposition \ref{odd}, which means that $PAF$ is limited controversial, each preferred extension of $PAF$ is stable. Detailed proof can be found in \cite{dung1995acceptability}.
\end{proof}

\begin{proposition}
Given $PAF=(\mathcal{A}, \mathcal{D})$, if there exists an ordinary argument $\langle +v_a, \lambda_a \rangle$ such that for all $\langle +v_b, \lambda_b \rangle \in \mathcal{A}_o$ and $\langle -v_b, \lnot \lambda_b \rangle \in \mathcal{A}_b$ it is the case that $v_b \precsim v_a$ and $\langle +v_a, \lambda_a \rangle$ is not in any cycle, then $\mathcal{E}_{prf}(PAF) = \mathcal{E}_{grd}(PAF)$.
\end{proposition}
\begin{proof}
Because $v_b \precsim v_a$ and $\langle +v_a, \lambda_a \rangle$ is not in any cycle, argument $\langle +v_a, \lambda_a \rangle$ does not receive any defeats. So the grounded extension is not an empty set. Suppose $\mathcal{E}_{prf}(PAF) \not= \mathcal{E}_{grd}(PAF)$, which means that there are more than one preferred extensions. Since $\langle +v_a, \lambda_a \rangle$ is contained in the grounded extension, it should also be contained in each preferred extension. However, each preferred extension indicates a distinct plan, which will be later proved by Proposition \ref{one} and its implication. Contradiction!
\end{proof}

The justification of optimal plans is then defined under various semantics in Definition \ref{PAFsemantics}. Similarly to \cite{liao2019prioritized}, we write $\concl{\langle +v, \lambda \rangle}$ for the conclusion $\lambda$ of an ordinary argument, and $\oplans{PAF}{sem}$ for the set of conclusions of ordinary arguments from the extensions under a specific semantics.
\begin{definition}[Optimal Plans]\label{oplans}
Given $PAF=(\mathcal{A}, \mathcal{D})$, a set of optimal plans, written as $\oplans{PAF}{sem}$, are the conclusions of the ordinary arguments within extensions.
\begin{multline*}
\oplans{PAF}{sem}= \\
\{\concl{\langle +v, \lambda \rangle} \mid \langle +v, \lambda \rangle \in E \text{ and } E \in \mathcal{E}_{sem}(PAF)  \}
\end{multline*}
\end{definition}
We show that the results of our approach are consistent with the rationality of decision-making through the following propositions. Firstly, all the accepted arguments within an extension indicate the same plan.
\begin{proposition}\label{one}
Given a plan-based argumentation framework $PAF = (\mathcal{A}, \mathcal{D})$ and an extension $E$ of $PAF$ under a specific semantics as defined in Definition \ref{PAFsemantics}, 
\begin{enumerate}
\item for any two ordinary arguments $\langle +v_a, \lambda_a \rangle, \langle +v_b, \lambda_b \rangle \in E$, it is the case that $\lambda_a = \lambda_b$;
\item for any ordinary argument $\langle +v_a, \lambda_a \rangle \in E$ and any blocking argument $\langle -v_b, \lnot \lambda_b \rangle \in E$, $\lambda_a \not = \lambda_b$.
\end{enumerate}
\end{proposition}
\begin{proof}
For any extension $E$ under a specific semantics, it is required that all the arguments in $E$ should be conflict-free. 1. By Definition \ref{pdefeat}, we can derive two cases: either there is no attack between these two arguments, or one argument attacks the other but does not succeed as a defeat due to the preference between two values from the arguments. For the former case, two arguments contain the same plan. For the latter case, since any attack between two arguments is mutual, if an attack from argument $\langle +v_a, \lambda_a \rangle$ to argument $\langle +v_b, \lambda_b \rangle$ fails to be a defeat due to the preference between two values from the arguments, the attack from argument $\langle +v_b, \lambda_b \rangle$ to argument $\langle +v_a, \lambda_a \rangle$ will succeed to be a defeat. That means that the second case is impossible and only the first case holds. Hence, the two arguments have the same plan. 2. We can prove in a similar way that for any ordinary argument $\langle +v_a, \lambda_a \rangle \in E$ and any blocking argument $\langle -v_b, \lnot \lambda_b \rangle \in E$, $\lambda_a \not = \lambda_b$,
\end{proof}
From that we can see, if there are multiple preferred extensions, then each of them indicates a distinct plan. Secondly, our argumentation-based approach always selects the plan through which the most preferred value gets promoted and does not select the plan through which the most preferred value gets demoted.
\begin{proposition}\label{forward}
Given a plan-based argumentation framework $PAF = (\mathcal{A}, \mathcal{D})$, let $v_a \in \operatorname{Val}$ be a value such that for all $\langle +v_b, \lambda_b \rangle \in \mathcal{A}_o$ and $\langle -v_b, \lnot \lambda_b \rangle \in \mathcal{A}_b$ it is the case that $v_b \precsim v_a$, then an argument with value $v_a$ is in a preferred extension. Typically, if it is not in a cycle, then it is in the grounded extension.
\end{proposition}
\begin{proof}
Because $v_b \precsim v_a$, according to Definition \ref{pdefeat}, an argument with $v_a$ only gets defeated by an argument with $v_a \sim v_b$ or $v_a = v_b$. In such a case, the defeats are mutual so $\langle +v_a, \lambda_a \rangle$ is self-defended. Thus, it is contained in a preferred extension. If it is not in a cycle, which means that it is not self-defended, then it is in the grounded extension. 
\end{proof}
Because of the above two propositions, the agent can conclude to follow an optimal plan to achieve his goal. Besides, the notion of optimal plans is defined as the set of conclusions of ordinary arguments from the extensions, so the set of optimal plans becomes empty if an extension does not contain any ordinary arguments. The following proposition indicates the conditions for which the set of optimal plans is not empty.
\begin{proposition}
Given a plan-based argumentation framework $PAF = (\mathcal{A}, \mathcal{D})$, $\oplans{PAF}{sem} \not= \emptyset$ iff there exists an ordinary argument $\langle +v_a, \lambda_a \rangle$ such that it is not defeated by a blocking argument $\langle -v_b, \lnot \lambda_b \rangle$ with $v_a \prec v_b$.
\end{proposition}
\begin{proof}
Having $\oplans{PAF}{sem} \not= \emptyset$ means that there is at least one extension which contains at least one ordinary argument. $\Rightarrow$: Suppose there does not exists an ordinary argument $\langle +v_a, \lambda_a \rangle$ such that it is not defeated by a blocking argument $\langle -v_b, \lnot \lambda_b \rangle$ with $v_a \prec v_b$, which means that all the ordinary arguments (if exist) are defeated by a blocking argument and not self-defended against a blocking argument. In such a case, there exists a blocking argument that does not receive any defeats, which makes all the ordinary arguments rejected. Contradiction! $\Leftarrow$: If there exists an ordinary argument such that it is not defeated by a blocking argument or self-defended against a blocking argument, then (1) the ordinary argument does not receive any defeats and thus it should be contained in the grounded extension, or (2) the ordinary argument is in a cycle with a blocking argument and thus it should be contained in a preferred extension, or (3) the ordinary argument receives defeats from other ordinary arguments and thus there is always an ordinary argument accepted. Hence, $\oplans{PAF}{sem}$ is not an empty set.
\end{proof}

\begin{example}
The plan-based argumentation framework $PAF$ can be represented as Figure. \ref{Drawing4}. Because
\begin{align*}
\mathcal{E}_{prf}(PAF) = &\mathcal{E}_{grd}(PAF) = \mathcal{E}_{stb}(PAF) =\\
&\{\{\langle +pv, (\alpha_2, \alpha_4, \alpha_5) \rangle, \langle +sf, (\alpha_2, \alpha_4, \alpha_5)\rangle,\\
& \langle -pv, \lnot (\alpha_1, \alpha_6) \rangle, \langle -sf, \lnot (\alpha_2, \alpha_3) \rangle \}\}
\end{align*}
and thus $\oplans{PAF}{sem} = \{(\alpha_2, \alpha_4, \alpha_5)\}$, the agent can follow plan $(\alpha_2, \alpha_4, \alpha_5)$ to get to a pharmacy.
\end{example}

\section{Conclusions}
In this paper, we developed a logic-based framework that combines modal logic and argumentation for value-based practical reasoning. Modal logic is used as a technique to represent and verify whether a plan with its local properties of value promotion or demotion can be followed to achieve an agent's goal. Seeing a verification result as an argument and defining a defeat relation based on an attack relation and preference over values, we then proposed an argumentation-based approach that allows an agent to reason about his plans using the verification results. Thus, our framework not only offers an approach for value-based practical reasoning with plans, but also makes a bridge between modal logic and argumentation in terms of argument construction. In the future, we would like to extend our framework by allowing an agent to have multiple goals instead of one goal as we assumed, or taking the actions of other agents into account in the context of multi-agent systems. More interestingly, we can study how autonomous agents are properly aligned with human values through adding constraints to the decision-making mechanism presented in this paper.

\bibliographystyle{named}
\bibliography{ijcai22}

\end{document}